\documentclass[letterpaper]{article} % DO NOT CHANGE THIS
\usepackage[preprint]{aaai2027}  % DO NOT CHANGE THIS
\usepackage[hyphens]{url}  % DO NOT CHANGE THIS
\usepackage{graphicx} % DO NOT CHANGE THIS
\usepackage{natbib}  % DO NOT CHANGE THIS AND DO NOT ADD ANY OPTIONS TO IT
\usepackage{caption} % DO NOT CHANGE THIS AND DO NOT ADD ANY OPTIONS TO IT
\usepackage{algorithm}
\usepackage{algorithmic}

\usepackage{newfloat}
\usepackage{listings}
\DeclareCaptionStyle{ruled}{labelfont=normalfont,labelsep=colon,strut=off} % DO NOT CHANGE THIS
\floatstyle{ruled}
\newfloat{listing}{tb}{lst}{}
\floatname{listing}{Listing}

\usepackage{booktabs}

\usepackage{mathtools} % amsmath with fixes and additions
\usepackage{booktabs} % commands to create good-looking tables
\usepackage{tikz} % nice language for creating drawings and diagrams
\usepackage[capitalize,noabbrev]{cleveref} % automatically adds type of the reference
\usepackage{amsthm}
\usepackage{thmtools} 
\usepackage{thm-restate}

\usepackage{lipsum}

\usepackage{subcaption}

\usepackage{listings}

\usepackage{tikz}
\usepackage{tcolorbox}

\usepackage{tabularx}

\usepackage{amsmath,amsfonts,bm}

\def\1{\bm{1}}
\newcommand{\train}{\mathcal{D}}

\def\eps{{\epsilon}}

\def\vk{{\bm{k}}}

\def\vw{{\bm{w}}}
\def\vx{{\bm{x}}}
\def\vy{{\bm{y}}}

\def\mB{{\bm{B}}}

\def\mI{{\bm{I}}}

\def\mK{{\bm{K}}}

\DeclareMathAlphabet{\mathsfit}{\encodingdefault}{\sfdefault}{m}{sl}
\SetMathAlphabet{\mathsfit}{bold}{\encodingdefault}{\sfdefault}{bx}{n}

\def\gE{{\mathcal{E}}}

\def\gN{{\mathcal{N}}}

\def\gS{{\mathcal{S}}}

\def\gX{{\mathcal{X}}}

\def\sR{{\mathbb{R}}}

\newcommand{\E}{\mathbb{E}}

\newcommand{\R}{\mathbb{R}}

\newcommand{\Cov}{\mathrm{Cov}}
\DeclareMathOperator*{\argmax}{arg\,max}

\newcommand{\D}{\mathcal{D}}

\newcommand{\xs}{\vx_*}
\newcommand{\f}[2]{f^{(#1)}_{#2}}
\newcommand{\set}[3]{\smash{\{#3\}_{#1}^{#2}}}
\newcommand{\K}{\mathbf{K}}
\newcommand{\Sobs}{\gS_\text{obs}}
\newcommand{\Shid}{\gS_\text{obs}^\mathsf{c}}
\newcommand{\fpar}{\smash{f^{(M)}_\parallel}}
\newcommand{\fperp}{\smash{f^{(M)}_\perp}}

\definecolor{onlinegreen}{RGB}{182, 215, 168} % Forest Green to match the left panel
\definecolor{offlinered}{RGB}{243, 204, 204}  % Firebrick Red to match the right panel
\definecolor{mydarkblue}{rgb}{0,0.08,0.45}

\newcommand{\tightcolorbox}[2]{%
  {\setlength{\fboxsep}{1pt}\colorbox{#1}{#2}}%
}

\definecolor{WOrchid}{RGB}{154, 100, 246}

\theoremstyle{plain}
\newtheorem{theorem}{Theorem}[section]

\theoremstyle{definition}
\newtheorem{definition}[theorem]{Definition}

\theoremstyle{remark}
\newtheorem{remark}[theorem]{Remark}

\crefname{proposition}{prop.}{props.}
\Crefname{proposition}{Prop.}{Props.}

\crefname{theorem}{thm.}{thms.}
\Crefname{theorem}{Thm.}{Thms.}

\title{Out-Of-The-Loop Multi-Fidelity Bayesian Optimization}
\author {
    Gustavo Sutter\textsuperscript{\rm 1,\rm 2,\rm \dag},
    Hao Wang\textsuperscript{\rm 3},
    Luis Ricardez-Sandoval\textsuperscript{\rm 3,\rm 4},
    Pascal Poupart\textsuperscript{\rm 1,\rm 2},
    Agustinus Kristiadi\textsuperscript{\rm 2,\rm 5}
}
\affiliations {
    \textsuperscript{\rm 1} Cheriton School of Computer Science, University of Waterloo, Waterloo, ON, Canada\\
    \textsuperscript{\rm 2} Vector Institute, Toronto, ON, Canada\\
    \textsuperscript{\rm 3} Department of Chemical Engineering, University of Waterloo, Waterloo, ON, Canada\\
    \textsuperscript{\rm 4} Waterloo Institute for Nanotechnology, University of Waterloo, Waterloo, ON, Canada\\
    \textsuperscript{\rm 5} Department of Computer Science, Western University, London, ON, Canada
}

\begin{document}

\maketitle

\renewcommand{\thefootnote}{\ensuremath{\dagger}}
\footnotetext{Corresponding author: \texttt{gsutterp@uwaterloo.ca}}
\renewcommand{\thefootnote}{\arabic{footnote}}

\begin{abstract}
Black-box optimization is a ubiquitous problem in science and engineering, often dealing with expensive objective functions with cheaper lower-fidelity proxies available.
Multi-fidelity Bayesian optimization (MF-BO) is a principled approach to this problem, leveraging correlations across different fidelities when querying the objective.
However, for many important MF-BO tasks, the true highest-fidelity function is prohibitively expensive to be part of the optimization loop.
Nevertheless, practitioners often have gold standard data (observations of the highest-fidelity function) obtained from previous experiments that might provide information for the current task.
For instance, in molecular optimization, chemists often pick the top-$k$ candidate molecules using various computer simulations, and later reveal their \emph{true} objective function values.
In this work, we demonstrate the suboptimality of standard MF-BO algorithms in the real-world scenarios above, even under ideal assumptions.
Next, we mitigate this problem by incorporating historical high-fidelity data accompanied by task descriptors---which can be explicitly given or extracted from unstructured metadata.
We demonstrate the effectiveness of our methods on synthetic functions, as well as real-world problems in chemistry and hyperparameter optimization.
\end{abstract}

% Uncomment the following to link to your code, datasets, an extended version or similar.
% You must keep this block between (not within) the abstract and the main body of the paper.
% Make sure that you do not de-anonymize yourself with these links.
% \begin{links}
%     \link{Anonymized Code}{}
% \end{links}

\section{Introduction}
\label{sec:intro}

\begin{figure}[t]
    \centering
    \includegraphics[width=\linewidth]{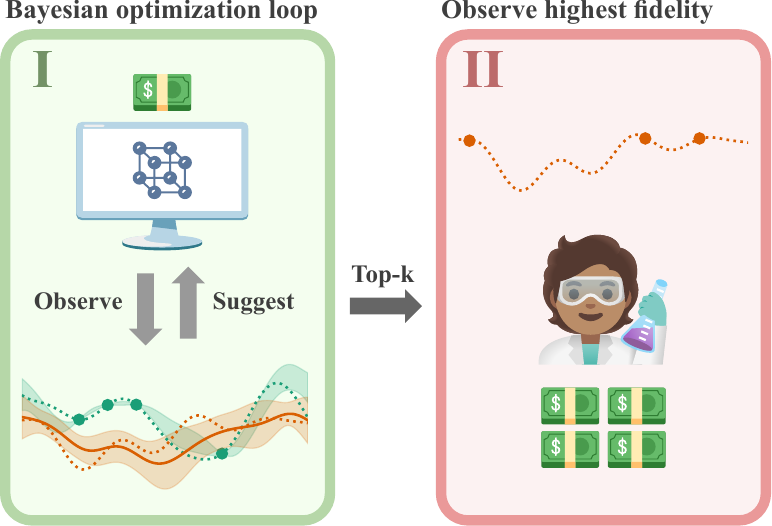}
    \caption{Visualization of the out-of-the-loop multi-fidelity setting, emphasizing the distinction between the \tightcolorbox{onlinegreen}{in-the-loop} and \tightcolorbox{offlinered}{out-of-the-loop} fidelities. \textbf{(I)} The \tightcolorbox{onlinegreen}{online} Bayesian optimization loop actively queries lower-fidelity functions. \textbf{(II)} The \tightcolorbox{offlinered}{offline} stage takes the final top-$k$ suggestions \textit{out of the loop} to be observed on the highest-fidelity objective.} 
    \label{fig:one}
    \vspace{-0.85em}
\end{figure}

A wide array of problems in chemistry~\citep{griffiths2020constrained,greenaway2023alchemist,Muthyala2026bodenovo}, biology~\citep{romero2013navigating,ruberg2023bayesdrug,martens2025holisticbioprocessdevelopmentscales}, and engineering \citep{ament2023sustainable,feurer2022auto,lam2018boaerospace} consists of optimizing black-box functions.
Such objectives often have prohibitively large input domains and are expensive to compute.
This creates the need for data-efficient black-box optimization algorithms that are able to find the optimal input point by querying the objective function as few times as possible.
Bayesian optimization \citep[BO;][]{Mockus1975,garnett_bayesoptbook_2023} is a principled way to solve this problem following the principles of Bayesian decision theory.

One solution to expensive objective function evaluations is the use of functions with different cost-fidelity trade-offs as proxies for the true, prohibitive objective. 
This setting is addressed by multi-fidelity Bayesian optimization \citep[MF-BO;][]{kennedy2000predicting,Huang2006cokriging} algorithms, where costs and correlations between proxies and the true objective are considered when selecting the next point to query in the BO loop.
For example, one can employ numerical simulators or machine learning models instead of performing wet-lab experiments~\citep{guan2022mlcatalysis}.

However, in many applications, the highest fidelity level is prohibitively expensive, so practitioners often do not include it in the BO loop.
Instead, the highest level available during online exploration is not the exact gold standard that is being optimized.
Usually, practitioners run optimization on the best available fidelity and evaluate the obtained maximizers on the real objective function of interest.
Notably, the results of previous experiments are saved for future reference and experiments, which can be exploited by future optimization campaigns.
For example, the highest fidelity in material design is the physical experiment in the laboratory, but often BO is run considering density functional theory \citep[DFT;][]{Argaman2000DFT} as the objective, with lower fidelities corresponding to lighter DFT configurations and/or machine learning interatomic potential models, such as UMA~\citep{wood2025umafamilyuniversalmodels}.
Only once the BO algorithm is done, its top suggestions are used in lab experiments, the real objective functions are computed, and the results are saved.

In this work, we first characterize the suboptimality of standard MF-BO under this setting, even assuming knowledge of the true correlation between fidelities---which is not available in real-world applications.
This highlights the need for extra information when the highest fidelity is not available in the loop.
Next, we propose a solution based on transferring information from previous optimization tasks to address the suboptimality problem.
We incorporate information from offline data collected at the end of past experiments via a multi-task multi-fidelity kernel.
Importantly, the proposed deep kernel can leverage task-specific features that are either explicitly provided or extracted from unstructured metadata using modern foundation models~\citep{Bommasani2021FoundationModels}, making use of the rich textual context available in real-world applications.
We validate our approach on synthetic benchmarks as well as real-world applications in chemistry and hyperparameter tuning, demonstrating the importance of knowledge transfer in the OOL-MF-BO setting.

\begin{figure*}[t]
    \centering
    \begin{subfigure}{0.33\linewidth}
        % \centering
        \includegraphics[width=\linewidth]{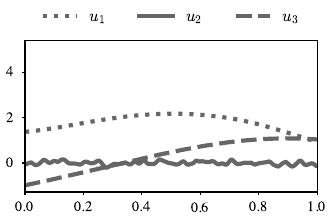}
        \caption{Latent processes}
    \end{subfigure}
    \begin{subfigure}{0.33\linewidth}
        % \centering
        \includegraphics[width=\linewidth]{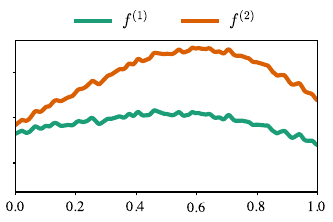}
        \caption{Objective functions}
    \end{subfigure}
    \begin{subfigure}{0.33\linewidth}
        % \centering
        \includegraphics[width=\linewidth]{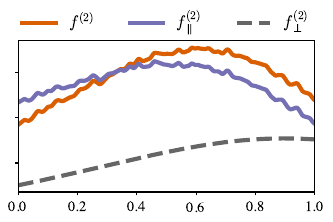}
        \caption{Target decomposition}  
    \end{subfigure}
    \caption{
    Illustration of the pitfalls of out-of-the-loop multi-fidelity optimization under the LMC kernel, with $M=2$. 
    \textbf{(a)} Three latent functions \(u_1, u_2, u_3\) sampled from GPs with different lengthscales.
    \textbf{(b)} Higher and lower fidelity functions, with $f^{(1)} = u_1 + u_2$ and $f^{(2)} = 2\,u_1 + u_2 + u_3$.
    That is, the process $u_3$ is present in the target but not in the observable fidelity.
    \textbf{(c)} The observable component $\smash{\f{2}{\|}}$ does not have the same maximizer as the target $\smash{\f{2}{}}$ due to the shift imposed by $\smash{\f{2}{\perp}}$.}
    \label{fig:thm}
\end{figure*}

The contributions of this work are as follows:
\begin{itemize}
    \item We are, to the best of our knowledge, the first to formalize the OOL-MF-BO problem.
    \item We introduce theoretical results characterizing the suboptimality of OOL-MF-BO even if the true correlations between fidelities are known.
    \item We introduce a method that leverages historical experimental data and task-specific features---extracted from either structured metadata or foundation models---to escape the aforementioned suboptimality, leading to a reduction in cumulative regret in multiple benchmarks.
    \item We show empirical results in the OOL-MF-BO setting across standard synthetic functions and on real-world chemistry and hyperparameter optimization benchmarks.
\end{itemize}

\section{Preliminaries}
\label{sec:background}

Let $f : \gX \rightarrow \sR$ denote an unknown objective function on a space \(\gX \subset \sR^d\).
The goal of black-box optimization is to find an optimal point $\xs \in \argmax_{\vx \in \gX} f(\vx)$ while assuming that $f$ is expensive to compute and the input domain cannot be explored exhaustively.
The objective function is accessed via noisy observations \(y = f(\vx) + \epsilon\) with \(\eps \sim \gN(0,\sigma_n^2)\), stored in the observation dataset \(\train_t = \set{i=1}{t}{(\vx_i,y_i)}\).

\subsection{Bayesian Optimization}

Bayesian optimization~\citep{Mockus1975,garnett_bayesoptbook_2023} provides a principled solution for black-box optimization problems based on two main components: a surrogate model \(p(f|\train_t)\) and an acquisition function \(\alpha(\vx;\train_t)\).
The surrogate model represents the current posterior belief over the unknown objective function, which is used by the acquisition function to evaluate candidate input locations and select the next query point according to \(\vx_{t+1} = \argmax_{\gX}\alpha(\vx;\train_t)\).
The optimization loop terminates after a predefined number of iterations \(T\) is reached.
Then, the algorithm returns its final choice \(\hat{\vx}_T = \argmax_{\vx_i \in \D_T} y_i\), given by the observed point with the highest objective value.

\paragraph{Multi-Fidelity Bayesian Optimization}

The multi-fidelity extension of the Bayesian optimization framework assumes that we have access to several functions in  \(\set{m=1}{M}{\f{m}{}:\gX \to \R}\) with the highest fidelity function \(\f{M}{}\) being the objective to be maximized, i.e., we aim to find 
\begin{equation*}
    \xs \in \argmax_{\vx \in \gX} \f{M}{}(\vx)
\end{equation*}
The observations dataset is a collection of triples \((\vx_i, y_i, m_i)\), where \(m_i \in \{1,...,M\}\) indicates the fidelity queried on step \(i\).
For each fidelity, there is an associated querying cost \(c^{(m)} \geq 0\) (assumed to be independent of \(\vx\)).
In the multi-fidelity regime, the acquisition function is given by \(\alpha(\vx, m;\train_t)\), which also incorporates the fidelity as an input.
The algorithm runs until the optimization budget \(\Lambda\) is exhausted. 
Both the surrogate model and the acquisition function are adapted to deal with a set of correlated functions.
Analogous to the single-fidelity case, at termination, the algorithm makes a final suggestion.
Usually, this is done by suggesting the observed point from the highest fidelity with the highest objective value, i.e. \(\hat{\vx}_T = \argmax_{\vx_i \in \D_{T,m}} y_i\) where \(\D_{T,m} = \{(\vx_i, y_i, m_i): m_i = M\}\). 
Another option is to suggest \(\hat{\vx}_T = \argmax_{\vx \in \gX} \E[\f{M}{}(\vx)|\train_T]\), allowing for points that were not observed in the highest fidelity level or not observed at all.

\subsection{Gaussian Processes}

Gaussian processes \citep[GPs;][]{Rasmussen_Williams_gp_2006} are the standard choice for surrogate modelling in BO.
They define a distribution over functions specified by a mean function \(\mu:\gX \to \R\) and a covariance kernel \(k:\gX \times \gX \to \R\).
Given a dataset \(\train_t\) with inputs and noisy observations, a GP prior \(f \sim \mathcal{GP}(\mu, k)\) induces a multivariate Gaussian distribution over the unseen function values and observed data \(\train_t\).
Conditioning on the data yields a GP posterior whose predictive distribution at a test point \(\vx\) is Gaussian with closed-form mean \(\mu_t(\vx)\) and variance \(\sigma_t(\vx)\).

\paragraph{Vector-valued Gaussian Process}
When dealing with vector-valued functions \(f:\gX \to \R^M\), it is still possible to use GPs.
Under this setting, it is assumed that \(f(x) = (f_1(x),\dots,f_M(x))^\top\) follows a vector-valued GP~\citep{alvarez2012vectorkernel},
fully characterized by a mean function \(\boldsymbol{\mu}:\gX \to \R^M\) and a matrix-valued kernel
\(\mathbf{K}:\gX \times \gX \to \R^{M \times M}\).
The entries \((\K(\vx,\vx'))_{m,m'}\) in the matrix \(\K(\vx,\vx')\) correspond to the covariance between \(\f{m}{}(\vx)\) and \(\f{m'}{}(\vx')\).

A popular choice is the linear model of coregionalization~\citep[LMC;][]{journel1978mining,goovaerts1997geostats}:
\begin{equation}
\label{eq:lmc}
[\K(\vx,\vx')]_{m,m'} = \sum_{q=1}^{Q} \mB^{(q)}_{m,m'} k_q(\vx, \vx'),
\end{equation}
where \(\mB^{(q)} \in \R^{M \times M}\) is a positive definite coregionalization matrix and \(k_q\) is a kernel for \(q \in \{1,\dots,Q\}\).
This corresponds to modelling each output as a linear combination of latent processes, \(\f{m}{}(\vx) = \sum_{q=1}^{Q} a^{(m)}_{q} u_q(\vx)\), where \(u_q \sim \mathcal{GP}(0,k_q)\) are i.i.d. latent processes.

When all latent processes have the same kernel, the resulting model is called an intrinsic coregionalization model ~\citep[ICM;][]{goovaerts1997geostats}, for which the covariance is:
\begin{equation}
\label{eq:icm}
[\K(\vx,\vx')]_{m,m'} = \mB_{m,m'} k(\vx, \vx'),
\end{equation}

Another popular kernel, the multi-information source (MISO) model of \citet{poloczek2017miso}, can be written in the LMC form by taking
$Q = M$ and setting \(\f{M}{} = u_1\) and \(\f{m}{} = u_1 + u_{m+1}\) for \(m = 1, \dots, M-1\).
Note that this requires distinct kernels for each process, placing it in the LMC family but outside ICM.
\section{Pitfalls of Out-of-the-Loop Fidelity}
\label{sec:pitfalls}

\begin{figure*}[!t]
    \centering
    \includegraphics[width=\linewidth]{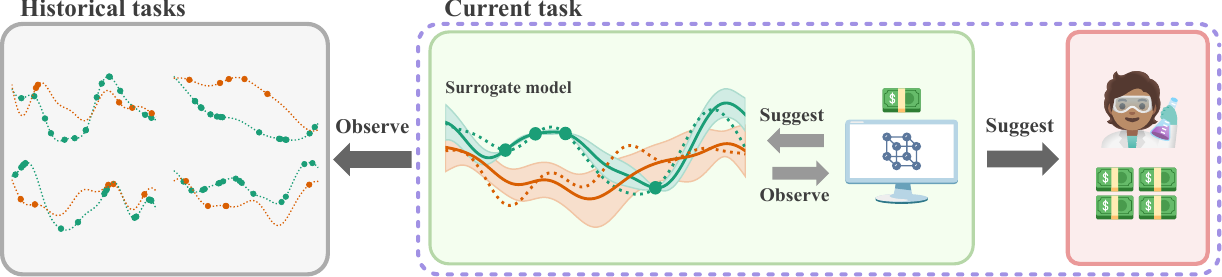}
    \caption{Proposed solution for the OOL-MF-BO setting based on data from historical tasks for which the highest fidelity was observed. The surrogate model has access to a dataset of historical tasks alongside the data being collected from the available fidelities of the current task.}
    \label{fig:task-transfer}
\end{figure*}

We start by characterizing what happens when we only have access to the lower fidelities in the BO loop.
As noted before, in many complex scientific applications the true objective function is prohibitively expensive---both computationally and financially---to be accessed within the loop.

Consider a multi-fidelity optimization problem with functions \(\set{m=1}{M}{\f{m}{}}\) drawn from a joint Gaussian process with zero mean and an LMC kernel, that is, each fidelity is a linear combination of the same set of latent processes \(\set{q=1}{Q}{u_q}\).
Let the costs \(\set{m=1}{M}{c^{(m)}}\) be such that \(c^{(M)} = \infty\), and let the budget be finite, \(\Lambda > 0\).
Therefore, we have \(m_i < M\) for all elements in \(\D_t = \{(\vx_i, y_i, m_i)\}_{i=1}^{N}\).
The final suggestion of the algorithm is given by
\begin{equation}
\label{eq:choice}
\hat{\vx}_T = \argmax_{\vx \in \gX} \E[\f{M}{}(\vx)\mid\train_T],
\end{equation}
since we are unable to query our target fidelity.

First, we partition the set of latent functions based on their presence in the lower fidelities.
Next, using the two sets in this partition, we decompose the objective function into observable and unobservable components.

\vspace{0.4em}
\begin{definition}
\label{def:part}
Let \(\Sobs = \{q\;|\;\smash{a^{(m)}_q} \neq 0 \;\text{for some}\; m < M\} \) denote the set of indices of latent processes that are observable in at least one of the lower fidelities.
In addition, define its complement \(\Shid = \{q\;|\;\smash{a^{(m)}_q} = 0 \;\text{for all}\; m < M\} \), capturing the latent processes that are not present in any observable fidelity.
\end{definition}
\vspace{0.4em}

\begin{definition}
\label{def:parandperp}
Let the target fidelity be \(\f{M}{}(\vx) = \smash{\sum_{q=1}^{Q} a^{(M)}_{q} u_q(\vx)}\).
Splitting this sum according to the partition \(\{\Sobs, \Shid\}\) of
\(\{1,\dots,Q\}\) from \Cref{def:part} yields
\[
\begin{aligned}
  \f{M}{} &= \fpar + \fperp, \quad\text{where}\\
  \fpar &:= \sum_{q \in \Sobs} a^{(M)}_q u_q,
  \quad
  \fperp := \sum_{q \in \Shid} a^{(M)}_q u_q .
\end{aligned}
\]
\end{definition}
\vspace{0.4em}

This decomposition of the objective function allows us to reason about what can be reconstructed from out-of-the-loop observations.
The key object to study under these conditions is the posterior mean \(\mu^{(M)}_t = \E[\f{M}{}\mid\D_t]\), which is used to make suggestions.

\vspace{0.4em}
\begin{restatable}{proposition}{propositionlmc}
\label{prop:lmc}
Under the LMC kernel defined in \eqref{eq:lmc}, the highest-fidelity posterior mean equals the posterior mean of the observable component: \(\mu^{(M)}_t = \E[\fpar \mid \D_t]\). The unobservable component \(\fperp\) contributes nothing.
\end{restatable}
\vspace{0.4em}

Crucially, this inability to capture information from \(\fperp\) is not resolved as more low-fidelity data is collected.
As a result, whenever the unobserved component shifts the location of the optimum, the algorithm incurs irreducible regret.

\vspace{0.4em}
\begin{restatable}{theorem}{theoremlmc}
\label{thm:lmc}
Assume \(\E[\fpar \mid \D_t] \to \fpar\) uniformly as \(t \to \infty\) and that the lower-fidelity observations become dense in \(\gX\).
Given the candidate \(\hat{\vx}_t\) selected via \eqref{eq:choice}, whenever \(\argmax \f{M}{} \neq \argmax \fpar\) the regret converges to a strictly positive constant: \( \lim\limits_{t \to \infty} (\f{M}{}(\xs) - \f{M}{}(\hat{\vx}_t)) > 0 \).
\end{restatable}
\vspace{0.4em}

\begin{remark}
\label{remark:miso}
Under MISO, every latent process appears in some observable fidelity (\(\Shid = \emptyset\)), therefore \(\smash{\fperp} = 0\) and \Cref{thm:lmc} does not hold.
However, this is a structural assumption on the fidelity hierarchy, and when it fails, the misspecified prior restores the suboptimality.
\end{remark}
\vspace{0.4em}

A concrete instance of the LMC problem is presented in \Cref{fig:thm}.
In this bi-fidelity example, the model has direct access to the true linear coefficients.
However, since the lower-fidelity function does not include the higher-frequency component, the posterior collapses to the observed component and the regret converges to a positive value.

The suboptimality of the LMC just presented relies on the set of non-observable latents being non-empty.
For the ICM setting, however, we show that the regret is likely irreducible even when \(\Shid = \emptyset\), as the target posterior mean converges to a fixed combination of the lower fidelities that in general does not share the optimum location with the objective.

\vspace{0.4em}
\begin{restatable}{corollary}{corollaryicm}
\label{cor:icm}
Under the ICM kernel defined in \eqref{eq:icm}, the highest-fidelity posterior mean corresponds to a fixed linear combination of the lower-fidelity posterior means, i.e., \(\mu^{(M)}_t(\vx) = \sum_{m=1}^{M-1} w_m\mu^{(m)}_t(\vx)\), where the weights \(\vw\) do not depend on the input location \(\vx\).
Furthermore, if \(\xs \neq \argmax_{\vx \in \gX} \sum_{m=1}^{M-1} w_m \f{m}{}(\vx)\), the simple regret converges to a strictly positive constant: \( \lim_{t \to \infty} (\f{M}{}(\xs) - \f{M}{}(\hat{\vx}_t)) > 0 \).
\end{restatable}
\vspace{0.4em}

This effect is clear in the bi-fidelity setting, which reduces to the high-fidelity posterior mean being a simple rescaling of the lower-fidelity one.
Therefore, whenever the lower-fidelity maximizer or minimizer is not aligned with the true optima, the model will have irreducible regret. 

\section{Handling Out-of-the-Loop Fidelity}

The root of the suboptimality stated in the previous section is the absence of the target fidelity information within the loop.
Therefore, to escape this regime, it is necessary to incorporate a higher-fidelity signal in the surrogate model.
As, by definition, the highest fidelity of the current objective is not available, the solution is to incorporate data from \emph{outside} the current task. 
Fortunately, in many science and engineering problems, practitioners do not start their optimization campaign in a vacuum.
Instead, they often possess a repository of historical data from previous experiments---datasets where the highest-fidelity outcomes from those experiments have already been observed and recorded.
Note, however, that the relationship between those previous experiments and the current experiment is often ambiguous, as past experiments may come from partially relevant configurations or settings.

Based on this observation, we propose a solution for the aforementioned problem in OOL-MF-BO.
Formally, consider \(N\) sets of functions \(\smash{\set{m=1}{M}{\f{m}{n}}}\) for \(n=1...N\), where the subscript \(n\) indicates the task index (i.e., previous experiments), and we are interested in optimizing task \(N\) (i.e., current experiment). 
Assume we have historical observations \(\mathcal{H}_n = \set{t=1}{T_n}{(\vx_t, y^{(m_t)},m_t)}\), for which the highest fidelity is available. 
We denote by \(\mathcal{H} := \set{n=1}{N-1}{\mathcal{H}_n}\) the set of all historical data.
In addition, assume each task has a task context vector \(e_n\) that the optimization algorithm can access.
We assume that the previous tasks cannot be queried; only their past observations are available.

\subsection{Surrogate modelling}

To incorporate the information from the previous tasks, we replace the fidelity kernel to also include tasks:
\begin{equation*}
\begin{aligned}
\label{eq:new-kernel}
\Cov(\f{m}{n}(\vx), &\f{m'}{n'}(\vx')) = k((\vx, m, n),(\vx', m', n')) \\
&= k_\text{data}(\vx, \vx') \cdot k_\theta((m, n),(m', n')). \\
\end{aligned}
\end{equation*}
The newly introduced term \(k_\theta((m, n),(m', n'))\) is a deep kernel~\citep{wilson2016deep} that operates on fidelity indices and task features:
\begin{equation*}
k_\theta((m, n),(m', n')) = k_\text{RBF}(\gE(m, e_n;\theta),\gE(m', e_{n'};\theta)),
\end{equation*}
where \(\gE\) is a neural network parametrized by \(\theta\).

We use a deep kernel as it can flexibly learn complex, nonlinear relationships between tasks and fidelities. This is particularly relevant in our setting, as the task features---whether manually defined or extracted via foundation models---can interact with fidelity indices in nontrivial ways.
This allows the model to reliably identify which information from the previous tasks is relevant to the current objective, as not all previous data is equally relevant.

\Cref{fig:task-transfer} illustrates how offline data from previous tasks can improve the modelling of the incumbent function and break the linear combination effect previously stated.
Data from historical tasks is taken into account in the modelling of the current objective and its lower-fidelity functions.

\subsection{Acquisition function}

The proposed method does not impose restrictions on the choice of acquisition function.
Once the historical dataset \(\mathcal{H}\) is considered in the surrogate model, any multi-fidelity acquisition function such as MF-MES~\citep{takeno2020mfmes} or MF-EI~\citep{Huang2006cokriging} can be used.
Importantly, the selected acquisition function should remain formulated to optimize the unobserved highest-fidelity objective.
At each step of the optimization loop, the next point and fidelity queried are given by 
\begin{equation*}
\vx_{t+1}, m_{t+1} = \argmax_{\vx \in \gX, m \in [M-1]} \alpha(\vx, m;\train_t,\mathcal{H}).
\end{equation*}
The main change from the usual MF-BO is the exclusion of the highest fidelity level from the optimization domain.

\subsection{Task Features}
\label{sec:task-feat}

Although some optimization problems may have natural task features, this is not the case for all problems.
In fact, many real-world problems can be better explained in words than through a hand-crafted feature vector.
A natural solution to this is to make use of foundation models trained on a wide variety of modalities and domains.
An example of a task description that we use in our experiments is provided in \Cref{ex:task-text}.
Notably, we extract only one vector per task, representing a negligible cost.

\begin{figure}[t]
    \centering
    \begin{tcolorbox}[colback=blue!3!white,colframe=gray!75!black, fontupper=\sffamily\small, left=2mm, right=2mm,title=Hydrocarbon task |  Solvation energy benchmark]
Context: Non-polar, hydrophobic solutes. Dominant interactions: weak Van der Waals dispersive forces. Features: Zero dipole moment, lack of hydrogen bond donors or acceptors.
Chemical space: Alkanes, alkenes, and aromatics. Solvation driven by cavity formation energy.
\end{tcolorbox}
\caption{Task description example from solvation energy benchmark.
To generate the task embeddings, the description is fed to an LLM that maps it to a fixed-dimensional vector.}
\label{ex:task-text}
\end{figure}

\section{Related work}

\begin{figure*}[ht]
    \centering
    \includegraphics[width=\linewidth]{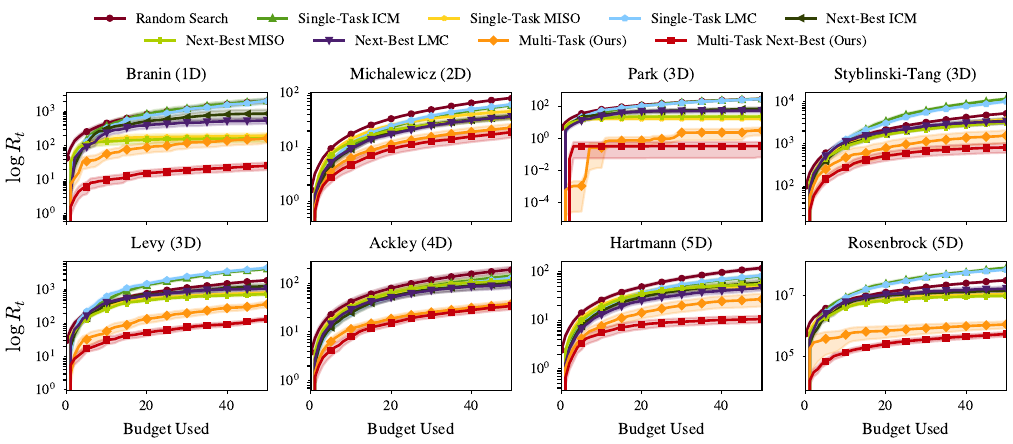}
    \caption{Performance comparison of our proposed methods on synthetic benchmark functions. Results are averaged over 20 independent trials with shading indicating the standard error.}
    \label{fig:res-synth}
\end{figure*}

\paragraph{Multi-task Bayesian Optimization}

Multi-task BO was originally proposed by \citet{swersky2013mtbo} as a solution to transfer knowledge across tasks to speed up the optimization of a related task.
In their work, they use an ICM kernel, learning the correlation between tasks directly by maximum likelihood estimation.
Alternatively, \citet{feng2020contextual} proposed the latent embedding multi-output kernel, in which the external task embeddings are fed to an RBF kernel function.
When the task embeddings are extracted by a neural network, this is closely related to deep kernel methods~\citep{wilson2016deep,zhang2025deepmfbo}; however, in this case, the network weights are frozen.
More recently, BOLT~\citep{zeng2025largescale} was introduced as a framework for using large language models (LLMs) to improve BO transfer across tasks. 
In their framework, an LLM is fine-tuned to produce a good initialization from task-specific context.

\citet{zhang2025multifidelity} is the only work to combine multi-fidelity and multi-task BO.
Their work differs from ours in three main ways: (i) the highest fidelity is always observable, (ii) the acquisition function used also weighs exploration for future tasks when selecting the next query point, and (iii) correlations among tasks arise from the i.i.d. assumption over task parameters rather than an explicit task kernel acting on the task contexts.

\paragraph{Bayesian Optimization Without Access to the Objective Function}
The central characteristic of our problem setting is not having access to the highest fidelity in the BO loop.
In contrast, prior work has examined other settings, characterized by different constraints and structural assumptions on the objective function.
\citet{zhang2025indirectquery} investigated the setting where, instead of directly selecting a point, one selects an action that is linked to the input space via a conditional distribution.
Related to the multi-fidelity setting, \citet{mikkola2023multi} explored the problem of unreliable fidelities and proposed a robust algorithm for MF-BO.
Furthermore, another related area of work is offline model-based optimization~\citep{kim2026offlineoptim}, which leverages powerful surrogate and generative models to optimize a function in a completely offline setting.
In contrast, we study the distinct challenge of OOL-MF-BO, in which the highest fidelity of the objective function is removed from the optimization loop, and lower fidelities are queried directly.

\section{Experiments}
\label{sec:experiments}

\begin{figure*}[t]
    \centering
    \includegraphics[width=\textwidth]{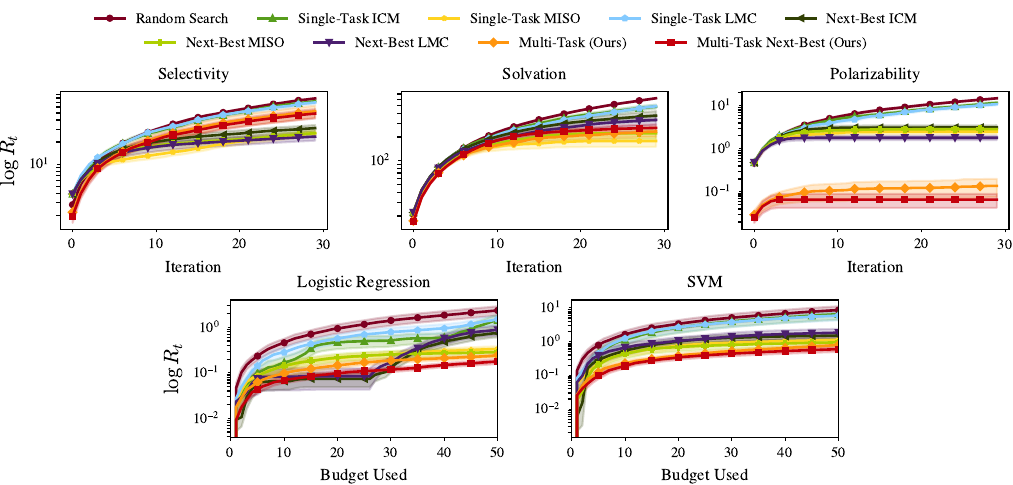}
    \caption{\textbf{(Top)} Performance comparison of all methods on three molecular optimization benchmarks: Xe/Kr selectivity, solvation energy, and polarizability. Results are averaged over 10 trials. 
    \textbf{(Bottom)} Cumulative regret curves for logistic regression and SVM design spaces. 
    Curves are averaged across 80 trials (8 datasets, each with 10 trials).}
    \label{fig:mfchem}
\end{figure*}

We perform OOL-MF-BO experiments on both synthetic functions, molecular benchmarks, and hyperparameter optimization tasks.
To evaluate performance, we visualize the cumulative regret at each step, i.e., \(R_t = \sum_{\tau=1}^{t} \f{M}{}(\xs) - \f{M}{}(\hat{\vx}_\tau)\).
We consider the following methods:
\begin{itemize}
    \item \textsc{Random Search}: Suggestions are made randomly by picking \(\hat{\vx}_t \sim \text{Uniform}(\gX) \).
    \item \textsc{Next-Best}: The surrogate models the current task up to the highest observable fidelity. The suggested point is given by \(\hat{\vx}_t = \argmax_{\vx \in \gX} \E[\f{M-1}{}(\vx)|\train_t]\).
    \item \textsc{Single-Task}: Surrogate models all fidelities of the current task. The suggested point is given by \(\hat{\vx}_t = \argmax_{\vx \in \gX} \E[\f{M}{}(\vx)|\train_t]\).
    \item \textsc{Multi-Task}: Our proposed method in which the out-of-the-loop fidelity is maximized. The suggested point is given by \(\hat{\vx}_t = \argmax_{\vx \in \gX} \E[\f{M}{}(\vx)|\train_t, \mathcal{H}]\).
    \item \textsc{Multi-Task Next-Best}: Variation of our proposed method in which the next-best fidelity is maximized. The suggested point is given by \(\hat{\vx}_t = \argmax_{\vx \in \gX} \E[\f{M-1}{}(\vx)|\train_t, \mathcal{H}]\).
\end{itemize}

For all methods, except \textsc{Random Search}, we use a GP with a constant mean and an RBF covariance kernel scaled by the dimension~\citep{hvarfner2024vanilla} as the data kernel.
For \textsc{Single-Task} and \textsc{Next-Best}, we show results with ICM, LMC, and MISO, covering popular multi-fidelity kernels used in practice.
For the \textsc{Next-Best} variants, Appendix~\ref{apdx:additional} repeats these experiments using the observed incumbent rather than the posterior-mean maximizer, with no qualitative change.
The deep kernel network \(\gE\) is parameterized as a two-layer network with 16 hidden units and an 8-dimensional output.
The acquisition function used is MF-MES~\citep{takeno2020mfmes}.
All experiments are implemented using the BoTorch package~\citep{Balandat2019BoTorchPB}.

\subsection{Synthetic Functions}

We now evaluate the methods on synthetic functions commonly used in the MF-BO literature.
Specifically, given a set of $M$ multi-fidelity objective functions $g^{(m)} : \mathbb{R}^d \to \mathbb{R}$ for $m=1, \dots, M$, we generate a family of related tasks by sectioning the original functions along a specific input dimension. 
Let the input space be decomposed such that the first $d'=d-1$ dimensions represent the design variables $\vx \in \mathbb{R}^{d-1}$, and the $d$-th dimension represents a task-defining parameter. By selecting a discrete set of task values $\mathcal{Z} = \{z_1, \dots, z_N\}$, we define the $n$-th task at the $m$-th fidelity level, $f^{(m)}_n : \mathbb{R}^{d-1} \to \mathbb{R}$, as \(\smash{\f{m}{n}}(\vx) = \smash{g^{(m)}}(x_1, \dots, x_{d-1}, z_n)\).
In this framework, $z_n$ serves as a contextual feature. This construction ensures 
that tasks are naturally correlated, as they are derived from slices of the 
same underlying $d$-dimensional response surface $g^{(m)}$.

Under this setting, we perform experiments on multi-fidelity versions of eight synthetic functions\footnote{\url{https://www.sfu.ca/~ssurjano/optimization.html}}: Branin, Michalewicz, Park, Styblinski-Tang, Levy, Ackley, Rosenbrock, and Hartmann.
For all functions, we perform 20 trials, in each trial sampling \(M \sim \text{Uniform}(\{2,3,4,5\})\), \(N \sim \text{Uniform}(\{2,3,4,5\})\), \(|H_n| \sim \text{Uniform}(\{20,30,40,50\})\).
Given the number of tasks, the task-defining parameters are sampled uniformly.
The costs are set to \(c^{(m)}=2^{m-1}\) for \(m \leq M-1\) and \(c^{(M)} = \infty\), capturing a situation where lowering the fidelity corresponds to halving the data size or tolerance parameter~\citep{eggensperger2021hpobench}.
Points in the historical observations are sampled uniformly with fidelity ratio \({2^{-(m-1)}}\).
More information is presented in Appendix~\ref{apdx:details-synth}.

The results are presented in Figure~\ref{fig:res-synth}.
Across all functions, \textsc{Multi-Task} approaches improve the results and are able to perform better than \textsc{Next-Best} and \textsc{Single-Task} methods.
Interestingly, the \textsc{Multi-Task Next-Best} method is consistently better than its single-task counterpart.
This gap in performance shows the clear advantage of incorporating previous tasks in the OOL-MF-BO setting.
As expected, the \textsc{Single-Task} method shows performance comparable to \textsc{Random Search} and, at best, on par with \textsc{Next-best}, highlighting the impossibility of optimizing the unavailable fidelity directly without auxiliary data. 
For further analysis of the different components of the setting, we refer to the ablation study exploring the effects of number of tasks, task relevance, and historical data size provided in Appendix~\ref{apdx:ablation}.

\subsection{Real-World Applications}

\paragraph{Molecular Optimization}
We experiment with three molecular optimization benchmarks, adapting the MF-BO problems used in \citet{Sabanza-Gil2025} to our OOL-MF-BO setting.
The problems are (i) Xe/Kr selectivity~\citep{gantzler2023mfbo}, (ii) solvation energy~\citep{mobley2014freesolv}, and (iii) polarizability~\citep{ghahremanpour2018alexandria}.
For all benchmarks \(M=2\) and the number of steps is set to \(T=30\) for all problems.
We use the Qwen3 Embedding 8B~\citep{zhang2025qwen3embedding} model to extract task features from text descriptions of the tasks.
\Cref{ex:task-text} shows one of the textual descriptions used.
Refer to Appendix~\ref{apdx:details-real} for information on the benchmarks and tasks.

\paragraph{Hyperparameter Optimization}
We evaluate the method on two hyperparameter optimization (HPO) benchmarks from HPOBench \citep{eggensperger2021hpobench}: (i) logistic regression (LR) and (ii) support vector machine (SVM).
For each benchmark, the hyperparameter space is provided alongside the validation accuracy for different fidelity levels, with $M=5$ and $M=3$ for LR and SVM, respectively.
Following the complexity scaling of the benchmark fidelities, costs are \(c^{(m)} = 3c^{(m-1)}\) with \(c^{(1)} =1\) and the budget is set to \( \Lambda=50\).
In both benchmarks, tasks correspond to different training and validation datasets, for a total of \(N=8\) tasks.
We perform 10 trials per dataset in the role of the current task, randomly sampling only 3 of the remaining tasks to be present in the historical data for each trial.
Refer to  Appendix~\ref{apdx:details-real} for more information.

The cumulative regret results for both families of benchmarks are presented in \cref{fig:mfchem}.
Our proposed methods \textsc{Multi-Task} and \textsc{Multi-Task Next Best} show strong performance, obtaining the lowest cumulative regrets for three out of five benchmarks.
In the polarizability benchmark, the observed gains start from the strong initialization due to the knowledge transfer from historical tasks and are maintained thereafter---resulting in significantly lower cumulative regret throughout the iterations.
On the other hand, on both HPOBench tasks (LR and SVM) our methods show consistently lower instantaneous regret during the entire run, leading to better results.

The benefits from historical data are clear when comparing \textsc{Multi-Task} and \textsc{Multi-Task Next Best} with the ICM versions of \textsc{Single Task} and \textsc{Next Best}.
Even though our proposed deep kernel has a structure very similar to the ICM kernel, namely a separable data kernel, the inclusion of historical information makes the results go from equivalent to \textsc{Random Search} to competitive in all evaluated benchmarks.

The strong performance of the MISO-based methods follows from \Cref{remark:miso}, and \textsc{Next-Best} LMC on Xe/Kr selectivity is the same effect in another guise: when the maximizers of the target and next-best fidelities are closely aligned, a method that refuses to extrapolate is difficult to improve upon.
Their subpar performance on the HPO and synthetic benchmarks shows that this condition is neither universal nor checkable in advance.
Gains from modelling the out-of-the-loop fidelity directly are thus realized precisely in the regime characterized in \Cref{sec:pitfalls}.

\section{Conclusion}
In this work, we explored the out-of-the-loop multi-fidelity Bayesian optimization setting.
We started from its formal definition and proceeded by characterizing the suboptimality of tackling problems of this family using standard MF-BO methods with both LMC and ICM---even if the fidelity covariance matrix is known.
Next, we demonstrated how historical data with the highest fidelity information can be used to improve the performance in the OOL-MF-BO setting.
We showed the benefits of the proposed method across both synthetic and real-world objective functions, demonstrating the effectiveness of our approach.

\paragraph{Limitations and Future Work}
Our analysis of the OOL suboptimality is developed for the LMC family of surrogates, which, although broad, does not cover all possible kernels; extending it to more general covariance structures remains open. 
Furthermore, performance depends on the relevance of the available historical tasks, which in practice is hard to know in advance. 
A natural next step is to apply OOL-MF-BO to challenging science and engineering problems using domain-specific foundation models as low-fidelity functions, and to develop more efficient methods for handling historical data that avoid the cubic complexity of GPs.

\bibliography{bibliography}
\clearpage
\appendix

\section{Proofs}
\label{apdx:proofs}

\propositionlmc*
\begin{proof}
Using the decomposition of the objective function and the linearity of expectation, the posterior mean can be written as
\begin{equation*}
\begin{split}
 \mu^{(M)}_t &= \E[\f{M}{} \mid \D_t] \\
 &= \E[\fpar + \fperp \mid \D_t] \\
 &= \E[\fpar \mid \D_t] + \E[\fperp \mid \D_t].
\end{split}
\end{equation*}
By \Cref{def:parandperp}, the second term can be written as \(\E[\fperp \mid \D_t] = \sum_{q\in\Shid} a^{(M)}_{q}\E[u_q \mid \D_t]\); however, \(u_q \perp \D_t\) for all \(q \in \Shid\).
Therefore, \(\E[u_q \mid \D_t] = \E[u_q] = 0\), given that all latent processes have a zero-mean prior.
\end{proof}

\theoremlmc*
\begin{proof}
By \Cref{prop:lmc}, the posterior mean of the target fidelity evaluates exactly to the posterior mean of the observable component, meaning \(\E[\f{M}{}(\vx) \mid \D_t] = \E[\fpar(\vx) \mid \D_t]\).
Consequently, the candidate point selected at step \(t\) via \eqref{eq:choice} can be equivalently written as
\begin{equation*}
\hat{\vx}_t = \argmax_{\vx \in \gX} \E[\fpar(\vx) \mid \D_t].
\end{equation*}

By the assumption of uniform convergence, \(\E[\fpar \mid \D_t] \to \fpar\) as \(t \to \infty\). Since the lower-fidelity observations become dense in \(\gX\), the maximizer of the posterior mean converges to the maximizer of the limiting function. Let \(\vx^*_{\parallel} = \argmax_{\vx \in \gX} \fpar(\vx)\). It follows that
\begin{equation*}
\lim_{t \to \infty} \hat{\vx}_t = \vx^*_{\parallel}.
\end{equation*}

Let \(\xs = \argmax_{\vx \in \gX} \f{M}{}(\vx)\) be the true global optimum. We can now evaluate the limit of the simple regret as \(t \to \infty\). Assuming \(\f{M}{}\) is continuous, we have
\begin{equation*}
\lim_{t \to \infty} \left( \f{M}{}(\xs) - \f{M}{}(\hat{\vx}_t) \right) = \f{M}{}(\xs) - \f{M}{}(\vx^*_{\parallel}).
\end{equation*}

By the premise that \(\argmax \f{M}{} \neq \argmax \fpar\), the point \(\vx^*_{\parallel}\) is not a global maximizer of the true objective function \(\f{M}{}\). Therefore, the value of the target fidelity at the true optimum is strictly greater than its value at the observable optimum:
\begin{equation*}
\f{M}{}(\xs) > \f{M}{}(\vx^*_{\parallel}).
\end{equation*}
Subtracting \(\f{M}{}(\vx^*_{\parallel})\) from both sides yields \(\f{M}{}(\xs) - \f{M}{}(\vx^*_{\parallel}) > 0\), which concludes the proof that the simple regret converges to a strictly positive constant.
\end{proof}

\corollaryicm*
\begin{proof}
We begin by writing the general formula for the posterior mean for all \(m \leq M\):
\begin{equation*}
    \mu^{(m)}_t(\vx) = \vk_m(\vx)^\top (\mK + \sigma_n^2\mI)^{-1}\vy,
\end{equation*}
where \(\vk_m(\vx) = [\Cov(\f{m}{}(\vx), \f{m_i}{}(\vx_i))]_{i=1}^{N}\) is the cross-covariance vector, \(\mK = [\Cov(\f{m_i}{}(\vx_i),\f{m_j}{}(\vx_j))]_{i=1,j=1}^{N,N}\) is the covariance matrix of the observation dataset, and \(\vy = [y_i]_{i=1}^{N}\) is a vector collecting all observed responses.
Recall that, as we are dealing with the ICM kernel, \(\Cov(\f{m}{}(\vx), \f{m'}{}(\vx')) = \mB_{m, m'}k(\vx, \vx')\).

Let \(\vw = \mB_{<M,<M}^{-1}\mB_{<M,M}\), where \(\mB_{<M,<M}\) is the block of the coregionalization matrix corresponding to the observable fidelities and \(\mB_{<M,M}\) is the block of cross-terms between the observable and the non-observable fidelity.
Since \(\mB\) is positive definite, \(\mB_{<M,<M}^{-1}\), and hence \(\vw\), are well defined.

Note that this definition of \(\vw\) implies \(\sum_{m<M} w_m \mB_{m,m'} = \mB_{M,m'}\) for all observable \(m'\).
Because \(m_i < M\) holds for all points in the dataset \(\D_t\), we have
\begin{equation}
\label{eq:covrelation}
    \vk_M(\vx) = \sum_{m<M}w_m\vk_m(\vx).
\end{equation}
Finally, we write the posterior mean of the highest fidelity and substitute \eqref{eq:covrelation} for the cross-covariance term:
\begin{equation*}
\begin{split}
 \mu^{(M)}_t(\vx) &= \vk_M(\vx)^\top (\mK + \sigma_n^2\mI)^{-1}\vy \\
 &= \Big[\sum_{m<M}w_m\vk_m(\vx)\Big]^\top (\mK + \sigma_n^2\mI)^{-1}\vy \\
 &= \sum_{m<M}w_m\big[\vk_m(\vx)^\top (\mK + \sigma_n^2\mI)^{-1}\vy\big] \\
 &= \sum_{m<M}w_m\mu^{(m)}_t(\vx),
\end{split}
\end{equation*}
which concludes the first part of the proof.

It remains to establish the irreducible regret. Assume that, as \(t \to \infty\), each lower-fidelity posterior mean converges uniformly to its function, \(\mu^{(m)}_t \to \f{m}{}\) for all \(m < M\), which holds as the lower-fidelity observations become dense in \(\gX\). Combining this with the first part of the proof, the highest-fidelity posterior mean converges to the fixed combination
\begin{equation*}
\mu^{(M)}_t(\vx) \to \sum_{m<M} w_m \f{m}{}(\vx) =: g(\vx).
\end{equation*}
Consequently, the candidate point selected via \eqref{eq:choice} satisfies
\begin{equation*}
\lim_{t \to \infty} \hat{\vx}_t = \argmax_{\vx \in \gX} g(\vx) =: \vx^*_g.
\end{equation*}
Let \(\xs = \argmax_{\vx \in \gX} \f{M}{}(\vx)\) be the true global optimum. Assuming \(\f{M}{}\) is continuous, the limit of the simple regret is
\begin{equation*}
\lim_{t \to \infty} \left( \f{M}{}(\xs) - \f{M}{}(\hat{\vx}_t) \right) = \f{M}{}(\xs) - \f{M}{}(\vx^*_g).
\end{equation*}
By the premise \(\xs \neq \argmax_{\vx \in \gX} \sum_{m<M} w_m \f{m}{}(\vx) = \vx^*_g\), the point \(\vx^*_g\) is not a global maximizer of \(\f{M}{}\), so \(\f{M}{}(\xs) > \f{M}{}(\vx^*_g)\). Therefore \(\f{M}{}(\xs) - \f{M}{}(\vx^*_g) > 0\), concluding the proof.
\end{proof}
\section{Experimental Details}

\subsection{Synthetic Functions}
\label{apdx:details-synth}

Here, we provide the formula for all the multi-fidelity functions considered in our experiments.
Some multi-fidelity extensions are taken from previous work (indicated accordingly), while others are introduced in this study.

All functions have a continuous fidelity parameter \(m \in [0,1]\), where \(m=1\) recovers the original function.
For each benchmark, the fidelity is discretized into \(M\) randomly chosen levels.

Throughout, we define the fidelity shift
\[
\delta(m) = 0.1(1-m).
\]

\paragraph{Multi-fidelity Branin \citep{mikkola2023multi}:}

\begin{equation*}
\begin{aligned}
f(\vx,m)
&=
\Bigl(
x_2-(b-\delta(m))x_1^2
+cx_1-r
\Bigr)^2  \\
&\quad
+10(1-t)\cos(x_1)+10 ,
\end{aligned}
\end{equation*}

where
\[
\begin{gathered}
\vx\in[-5,10]\times[0,15],\\
b=\frac{5.1}{4\pi^2},\qquad
c=\frac{5}{\pi},\\
r=6,\qquad
t=\frac{1}{8\pi}.
\end{gathered}
\]

\paragraph{Multi-fidelity Hartmann \citep{mikkola2023multi}:}

\begin{equation*}
\begin{aligned}
f(\vx,m)
&=
-(\alpha_1-\delta(m))
\exp\!\Bigl(
-\sum_{j=1}^{6}
A_{1j}(x_j-P_{1j})^2
\Bigr)
\\
&\quad
-
\sum_{i=2}^{4}
\alpha_i
\exp\!\Bigl(
-\sum_{j=1}^{6}
A_{ij}(x_j-P_{ij})^2
\Bigr).
\end{aligned}
\end{equation*}

where
\[
\vx\in[0,1]^6,\qquad
\boldsymbol{\alpha}=(1.0,1.2,3.0,3.2),
\]
and \(A\) and \(P\) are the standard Hartmann constants.

\paragraph{Multi-fidelity Park \citep{xiong2013seqdesign}:}

\begin{equation*}
\begin{aligned}
f(\vx,m)
&=
\frac{x_1}{2}
\Biggl(
\sqrt{
1+\frac{x_4(x_2+x_3^2)}{x_1^2}
}
-1
\Biggr)
\\
&\quad
+\bigl(x_1+(3-1.5\delta(m))x_4\bigr)
\exp(1+\sin x_3).
\end{aligned}
\end{equation*}

where \(\vx\in[0,1]^4\).

\paragraph{Multi-fidelity Rosenbrock:}
(adapted from BoTorch \citep{Balandat2019BoTorchPB})

\begin{multline*}
f(\vx,m)
=
\sum_{i=1}^{d-2}
\Bigl[
100(x_{i+1}-x_i^2+\delta(m))^2
\\
+(x_i-1+\delta(m))^2
\Bigr].
\end{multline*}

where \(\vx\in[-5,10]^d\).

\paragraph{Multi-fidelity Ackley:}

\begin{equation*}
\begin{aligned}
f(\vx,m)
&=
-a\exp\!\Biggl(
-b
\sqrt{
\frac{1}{d-1}
\sum_{i=1}^{d-1}
(x_i-\delta(m))^2
}
\Biggr)
\\
&\quad
-\exp\!\Biggl(
\frac{1}{d-1}
\sum_{i=1}^{d-1}
\cos\!\bigl(c(x_i-\delta(m))\bigr)
\Biggr)
\\
&\quad
+a+e,
\end{aligned}
\end{equation*}

where
\[
\vx\in[-32.768,32.768]^d,\qquad
(a,b,c)=(20,0.2,2\pi).
\]

\paragraph{Multi-fidelity Levy:}

\begin{equation*}
\begin{aligned}
f(\vx,m)
&=
\sin^2(\pi w_1)
\\
&\quad
+
\sum_{i=1}^{d-2}
(w_i-1)^2
\bigl(1+10\sin^2(\pi w_i+1)\bigr)
\\
&\quad
+
(w_{d-1}-1)^2
\bigl(1+\sin^2(2\pi w_{d-1})\bigr),
\end{aligned}
\end{equation*}

where
\[
\vx\in[-10,10]^d,\qquad
w_i=1+\frac{x_i-1+\delta(m)}{4}.
\]

\paragraph{Multi-fidelity Michalewicz:}

\begin{equation*}
\begin{aligned}
f(\vx,m)
&=
-
\sum_{i=1}^{d-1}
\sin(x_i-\delta(m))
\\
&\qquad\times
\left[
\sin\!\left(
\frac{i(x_i-\delta(m))^2}{\pi}
\right)
\right]^{2A},
\end{aligned}
\end{equation*}

where
\[
\vx\in[0,\pi]^d,\qquad
A=10.
\]

\paragraph{Multi-fidelity Styblinski--Tang:}

\begin{multline*}
f(\vx,m)
=
\frac12
\sum_{i=1}^{d-1}
\Bigl[
(x_i-\delta(m))^4
\\
-16(x_i-\delta(m))^2
+5(x_i-\delta(m))
\Bigr].
\end{multline*}

where \(\vx\in[-5,5]^d\).

\subsection{Chemistry Applications}
\label{apdx:details-real}

The three benchmarks are taken from \citet{Sabanza-Gil2025}, where the high- and low-fidelity values are provided directly in tabular form.
For each dataset, we partition it into separate tasks and generate a textual description of each task using Gemini 3 Pro\footnote{\url{https://gemini.google.com/}}.
Such textual descriptions are passed to Qwen3 Embedding 8B \citep{zhang2025qwen3embedding}, producing task context vectors.
For the three benchmarks, we use 30 examples per historical task, sampled uniformly on each trial.
We now explain the partitions and provide the task descriptions for each benchmark.

\paragraph{Solvation energy \citep{mobley2014freesolv}:} For this dataset, we split the tasks based on the functional groups present in the molecules.
The four resulting tasks are Nitrogenous, Oxygenated, Halogenated, and Hydrocarbon.
We establish a priority hierarchy (Nitrogen > Oxygen > Halogens) to assign each molecule to one of four mutually exclusive tasks.
The following task descriptions are used:

\begin{tcolorbox}[colback=blue!3!white,colframe=gray!75!black,fontupper=\small]
    \textbf{Hydrocarbon:} ``Context: Non-polar, hydrophobic solutes. Dominant interactions: weak Van der Waals dispersive forces. Features: Zero dipole moment, lack of hydrogen bond donors or acceptors. Chemical space: Alkanes, alkenes, and aromatics. Solvation driven by cavity formation energy.''
    
    \textbf{Oxygenated:} ``Context: Polar, hydrophilic solutes. Dominant interactions: Strong hydrogen bonding and permanent dipole-dipole attraction. Features: High electronegativity difference, presence of hydroxyl (OH) or carbonyl (C=O) groups. Chemical space: Alcohols, ethers, ketones.''
    
    \textbf{Nitrogenous:} ``Context: Basic, polar solutes. Dominant interactions: Proton acceptance and electrostatic contributions. Features: Presence of lone pairs on Nitrogen, variable hybridization (sp2/sp3). Chemical space: Amines, amides, nitriles. Solvation driven by specific H-bond networks.''
    
    \textbf{Halogenated:} ``Context: Lipophilic, polarizable solutes. Dominant interactions: Electrostatic sigma-hole interactions and halogen bonding. Features: Large atomic radii, high surface area, 'soft' electron clouds. Chemical space: Alkyl halides. Solvation driven by entropy and polarizability.''

\end{tcolorbox}

\paragraph{Xe/Kr selectivity \citep{gantzler2023mfbo}:} We partition this dataset based on the pore diameter of each instance.
This results in 3 tasks: Small Pore (\(\mathring{A} < 15\), Medium Pore (\(15 \leq \mathring{A} \leq 25\)), and Large Pore (\(\mathring{A} > 25\)).
The following task descriptions are used:

\begin{tcolorbox}[colback=blue!3!white,colframe=gray!75!black,fontupper=\small]
    \textbf{Small Pore:} ``Context: High confinement regime (< 15 Angstroms). Mechanism: Steric sieving and overlap of potential energy surfaces. Geometry: Dense packing, high framework density, restricted void fraction. Separation driven by repulsive forces and size exclusion.''
    
    \textbf{Medium Pore:} ``Context: Optimal adsorption regime (15-25 Angstroms). Mechanism: Strong guest-host affinity without steric penalty. Geometry: Mesoporous, balanced void fraction. Separation driven by attractive Van der Waals wells and surface interactions.''
    
    \textbf{Large Pore:} ``Context: Bulk fluid regime (> 25 Angstroms). Mechanism: Diffusion-dominated transport, weak confinement. Geometry: Open frameworks, extremely high void fraction, low crystal density. Separation driven by pore wall surface area rather than pore size.''

\end{tcolorbox}

\paragraph{Polarizability \citep{ghahremanpour2018alexandria}:} For this dataset, we split the tasks based on the elemental composition of the molecules.
The four resulting tasks are Hetero-Halogens, CHN Compounds, CHO Compounds, and Hydrocarbons.
We established a priority hierarchy (Hetero-Halogens [F, Cl, Br, I, S, P] > Nitrogen > Oxygen) to assign each molecule to one of four mutually exclusive tasks.
The following task descriptions are used:

\begin{tcolorbox}[colback=blue!3!white,colframe=gray!75!black,fontupper=\small]
    \textbf{Hydrocarbon:} ``Context: Saturated carbon frameworks. Electronic environment: Hard electron clouds, high HOMO-LUMO gap. Polarizability source: Strictly volumetric scaling (size-dependent). Absence of permanent dipoles or lone pairs.''
    
    \textbf{CHO Compound:} ``Context: Oxygenated carbon frameworks. Electronic environment: High electronegativity contrast, distorted electron density. Polarizability source: Localized electron density around Oxygen atoms, permanent dipole contributions.''
    
    \textbf{CHN Compound:} ``Context: Nitrogenous carbon frameworks. Electronic environment: Lone pair donation, resonance effects. Polarizability source: Mobile pi-electrons and orbital hybridization changes (sp/sp2/sp3).''
    
    \textbf{Hetero-Halogens:} ``Context: Heavy atom substituted frameworks. Electronic environment: Diffuse, soft orbitals (d-orbital participation for S/P). Polarizability source: High deformability of large valence shells (I, Br, S, P). Significant Van der Waals radii overlap.''

\end{tcolorbox}

\begin{figure*}[t]
  \centering
  \includegraphics[width=\textwidth]{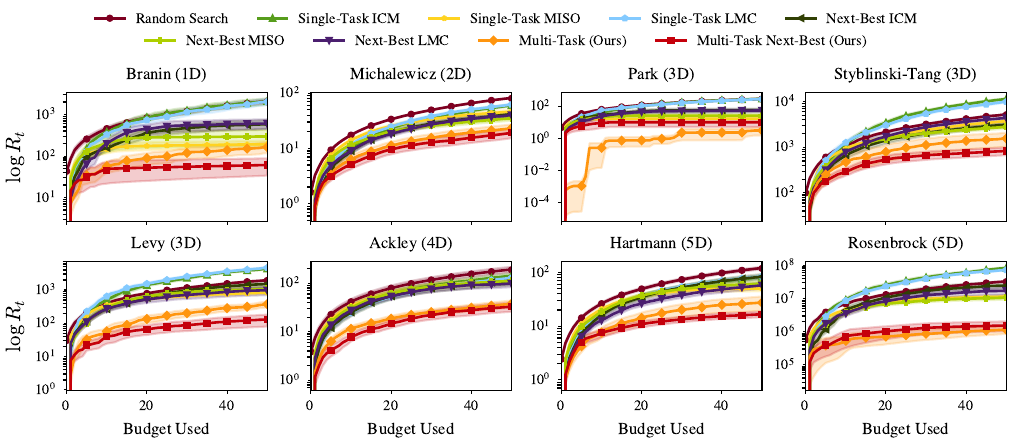}
  \caption{Synthetic benchmarks with the \textsc{Next-Best} methods reporting the best observed point instead of the posterior-mean maximizer. All other methods are unchanged from \cref{fig:res-synth}.}
  \label{fig:incumbent-synth}
\end{figure*}

\subsection{Hyperparameter Optimization Applications}
\label{apdx:details-real}

The logistic regression (LR) and support vector machine (SVM) from \citet{eggensperger2021hpobench} are used. 
The LR design space consists of the regularization strength $\alpha$ and the initial learning rate $\eta_0$, both within $[10^{-5},10^{0}]$; the SVM design space consists of the penalty parameter $C$ and the kernel coefficient $\gamma$, both within $[2^{-10},2^{10}]$.
All four hyperparameters are optimized in log space. 
Both benchmarks provide a single fidelity dimension, discretized by the original authors. For LR it is the number of SGD iterations, taking the five values $\{10, 37, 111, 333, 1000\}$, with the subsample fidelity fixed to its default value of $1$; for SVM it is the fraction of the training data used, taking the three values $\{1/9, 1/3, 1\}$.

We use the following 8 tasks: 10101, 53, 146818, 146821, 9952, 146822, 31, and 3917. 
Each task is represented by the meta-features of its underlying OpenML~\citep{Vanschoren2014openml} dataset, retaining only those meta-features that are available for all tasks and non-constant across them.
The resulting 20-dimensional feature vector is then standardized to zero mean and unit variance per dimension across the task set.
Each of the 8 tasks serves as the target task in turn, and for each run we uniformly sample 3 of the remaining 7 tasks as the historical tasks, so that no method is given access to the full task set at once.
Each historical task contributes 30 observations, uniformly distributed across its fidelity levels, and the target task is seeded with 5 observations at its lowest fidelity. The historical subset and all initial observations are resampled independently for each of the 10 runs.

\subsection{Computational Resources}

Experiments are implemented in Python using the BoTorch~\citep{Balandat2019BoTorchPB} library version 0.16.1.
The accompanying code contains all the software requirements and installation instructions.
All experiments are executed on a single NVIDIA L40S using \texttt{float64}, as suggested by BoTorch's official documentation.
1
\section{Additional Results}
\label{apdx:additional}

Throughout the paper, all methods report the maximizer of the posterior mean of their target fidelity.
An alternative convention is to report the best point observed so far, \(\hat{\vx}_t = \arg\max_{\vx_i \in \mathcal{D}_t} y_i\).
This rule is only available to methods that observe the fidelity they report on, i.e., the \textsc{Next-Best} variants; the remaining methods target the out-of-the-loop fidelity, which is never observed in the loop.

\Cref{fig:incumbent-synth,fig:incumbent-real} repeat the experiments of \Cref{sec:experiments} with the \textsc{Next-Best} methods using the observed incumbent.
All other methods and all experimental settings are unchanged.

\begin{figure*}[t]
  \centering
  \includegraphics[width=\textwidth]{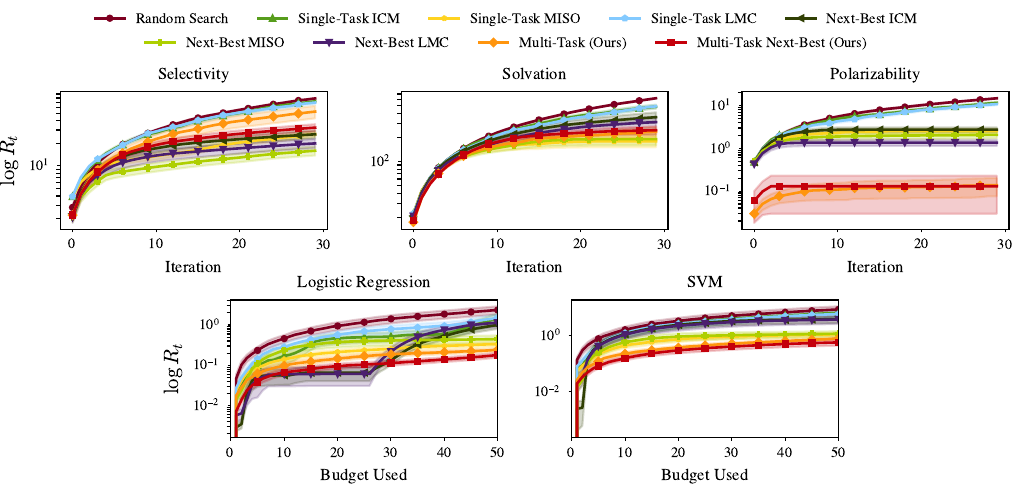}
  \caption{Real-world benchmarks under the observed-incumbent rule. All other methods are unchanged from \cref{fig:mfchem}.}
  \label{fig:incumbent-real}
\end{figure*}

\section{Ablation Studies}
\label{apdx:ablation}

\begin{figure*}[t]
    \centering
    \includegraphics[width=\linewidth]{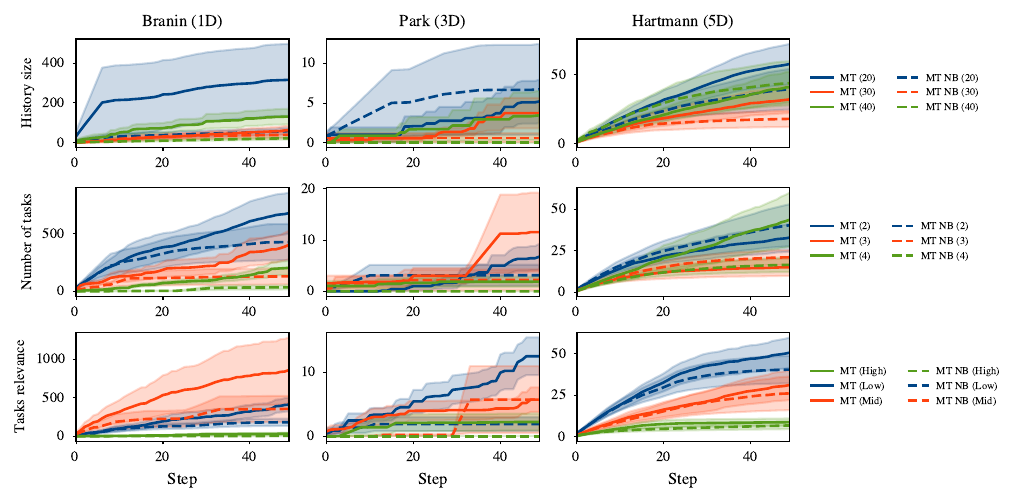}
    \caption{Results of the ablation study of our proposed methods, Multi-task (MT) and Multi-task Next-Best (MT-NB) on Branin, Park and Hartmann synthetic functions. We vary the historical dataset size (\textbf{top}), the number of tasks (\textbf{middle}), and historical tasks' relevance (\textbf{bottom}). Results are averaged over 10 independent trials.}
    \label{fig:ablation}
\end{figure*}

In this section, we evaluate the impact of different components of our method on the Branin, Park, and Hartmann synthetic functions described in Section~\ref{sec:experiments}. Unless stated otherwise, we use the same setting as before and perform 10 trials per configuration.

\paragraph{Historical dataset size}
We first analyze the effect of the historical dataset size (first row of \cref{fig:ablation}). We vary $|H_n| \in \{20,30,40\}$ for all previous tasks. For lower-dimensional problems (Branin and Park), increasing the amount of historical data consistently reduces cumulative regret. For Hartmann, however, the results are mixed and do not show a clear trend.

\paragraph{Number of tasks}
Next, we vary the number of available historical tasks, $N \in \{2,3,4\}$.
The task-defining parameters are sampled $z_n$ uniformly. For Branin and Park, having more historical tasks clearly improves performance. For Hartmann, the effect is less conclusive, likely because task diversity increases with dimensionality.

\paragraph{Auxiliary Task Relevance}
Finally, we examine the effect of task similarity. We sample the incumbent task parameter as $z_N \sim \text{Uniform}(0,1)$ and draw the remaining parameters $z_1,\dots,z_{N-1}$ under low, medium, or high relevance. Low relevance is defined by $p_{\text{low}}(z \mid z_N) \propto |z - z_N|^2$, medium by $p_{\text{mid}}(z \mid z_N) = \text{Uniform}(0,1)$, and high by $p_{\text{high}}(z \mid z_N) = \mathcal{N}(z_N, 0.1^2)$. As shown in the last row of \cref{fig:ablation}, highly relevant historical tasks substantially improve optimization performance. Moreover, as dimensionality increases, the differences between low, medium, and high relevance become more pronounced.

% Check whether the conference requires a reproducibility checklist to be included in the paper.
% If so, you can uncomment the following line and ajust the path to include it.
% \input{ReproducibilityChecklist.tex}

\end{document}